\documentclass{article}
\usepackage{amssymb,amsthm,amsmath,amsfonts}

\usepackage{microtype}

\usepackage{graphicx}
\usepackage{filecontents,lipsum}
\usepackage{enumitem}

\usepackage{booktabs}
\usepackage{multirow}
\usepackage{makecell}
\usepackage{color}

\theoremstyle{plain}
\begingroup
 
\newtheorem{theorem}{Theorem} 

\newtheorem{lemma}{Lemma}
\endgroup

\theoremstyle{definition}
 
\newtheorem{remark}[theorem]{Remark} 
\newtheorem{hyp}[theorem]{Assumption}

\theoremstyle{remark}
 
\numberwithin{equation}{section}
\numberwithin{figure}{section}
\begin{document}

\title{Convergence of backpropagation with momentum for network architectures with skip connections}

\author{Chirag Agarwal\textsuperscript{1}, Joe Klobusicky\textsuperscript{2}, Don Schonfeld\textsuperscript{1}}

\date{}
\maketitle

\begin{abstract}
We study a class of deep neural networks with architectures that form a directed
acyclic graph (DAG).   For backpropagation defined by gradient descent with
adaptive momentum, we show weights converge for a large class of nonlinear
activation functions.  The proof generalizes the results of Wu et al. (2008)
who showed convergence for a feed-forward network with one hidden layer.
For an example of the effectiveness of DAG architectures, we describe an
example of compression through an AutoEncoder, and compare against sequential
feed-forward networks under several metrics.
\end{abstract}



\medskip
\noindent

\footnotetext[1]
{Department of Electrical and Computer Engineering, University of
Illinois at Chicago, Chicago, IL}

\footnotetext[2]
{Department of Mathematical Sciences, Rensselaer Polytechnic Institute, 110
8th Street, Troy, NY 12180
Email: klobuj@rpi.edu}

\section{Introduction}
\label{sec:introduction}

Neural networks have recently enjoyed an acceleration in popularity, with
new research adding to several decades of foundational work. From multilayer
perceptron (MLP) networks to the more prominent recurrent neural networks
(RNNs) and convolutional neural networks (CNNs), neural networks have become
a dominant force in the fields of computer vision, speech recognition, and
machine translation \cite{rosenblatt1961principles}. Increase in computational
speed and data collection have legitimized the training of increasingly complex
deep networks. The flow of information from input to output is typically
performed in a strictly sequential feed-forward fashion, in which for a network
consisting of   $L$ layers,  nodes in the $i$th layer receive
input from the $(i-1)$st layer, compute an output for each neuron through
an activation function, and in turn use this output as an input for the $(i+1)$st
 layer. A natural extension to this network structure is the addition of
``skip connections" between layers.  Specifically, we are interested in the
class of architectures in which the network of connections form a directed
acyclic graph (DAG).  The defining property of a DAG is that it can always
be decomposed into a \textit{topological ordering} of $L$ layers, in which
nodes in layer $i$ may be connected to layer $j$, where $j>i$.  A skip connection
is a connection between nodes in layers $i$ and $j$, with $j>i+1$.
There has been an increasing interest in studying networks with skip connections
which skip a small number of layers, with examples including Deep Residual
Networks (ResNet) \cite{he2016deep},  Highway Networks \cite{srivastava2015training},
and FractalNets \cite{larsson2016fractalnet}.  ResNets, for instance,  use
``shortcut connections" in which a copy of previous layers is mapped through
an identity mapping to future layers. Kothari and Agyepong \cite{kothari1996lateral}
introduced  ``lateral connections" in the form of a chain, with each unit
in a hidden layer connected to the next. 
The full generality of neural networks for DAG architectures was considered
in \cite{huang2016densely}, which demonstrated superior performance of neural
networks, entitled DenseNets, under a wide variety of skip connections.

As an example of the efficacy of DAG architectures considered in \cite{huang2016densely},
we consider AutoEncoders, a class of neural networks which provide a means
of data compression. For an AutoEncoder, input data, such as a pixelated
image, is also the desired output for a neural network.  During an encoding
phase, input is compressed through several hidden layers before arriving
at a middle hidden layer, called the code, having dimension smaller than
the input.  The next phase is decoding, in which input from the code is fed
through several more hidden layers until arriving at the output, which is
of the same dimension as the input.  The goal of compression is to minimize
the difference between input data and output.  In \cite{agarwal2018crossencoder},
Agarwal et. al introduced CrossEncoders and demonstrated its superior performance
against AutoEncoders with no skip-connections. In Section \ref{sec:experiment},
we extend the previous results to include the MNIST and Olivetti faces public
datasets.  We validate our results against several commonly used compression
based performance metrics.

Our main theoretical result is the convergence of backpropagation with DAG
architectures using gradient descent with momentum.  It is well known that
feed-forward architectures converge under backpropagation, which is  essentially
gradient descent applied to an error function (see \cite{bishop1995neural},
for instance). Updates for weights in backpropagation may be generalized
to include a momentum term, which can help with increasing the convergence
rate \cite{rumelhart1987parallel}. Momentum can help with escaping local
minima, but concerns of overshooting require careful arguments for establishing
convergence. Formal arguments for convergence have so far been restricted
to simple classes of neural networks. Bhaya \cite{bhaya2004steepest} and
Torii \cite{torii2002stability} studied the convergence with backpropagation
using momentum under a linear activation function. Zhang et al.\cite{zhang2006}
 generalized convergence for a class of common nonlinear activation functions,
including sigmoids, for the case of a zero hidden layer networks. Wu et al.
\cite{wu2008convergence} further generalized to one layer by demonstrating
that error is monotonically decreasing under backpropagation iterations for
sufficiently small momentum terms. The addition of a hidden layer required
 \cite{wu2008convergence} to make the additional assumption of bounded weights
during the iteration procedure. 

It is not evident whether applying the methods of \cite{wu2008convergence}
would generalize to networks with several hidden layers and skip connections,
or if they  would require stronger assumptions on boundedness of weights
or the class of activation functions.  We show in Section \ref{appsect} 
that convergence indeed does hold, with similar assumptions to the proof
of convergence of one hidden layer. In Theorem \ref{bigtheo}, we  give the
key inequality for proving Theorem \ref{maint2}, a recursive form for increments
of error and output values of hidden layers after each iteration.  This estimate
allows us to show that for sufficiently small momentum parameters (including
the case of zero momentum), error decreases with each iteration. Our approach
to convergence is somewhat more explicit than the traditional proof of gradient
descent, which minimizes a loss function without considering network architecture.

\section{Architecture for a feed-forward network with cross-layer connectivity}
\label{sec:architecture}

In this section, we formally explain  DAG architectures, and the associated
backpropagation algorithm with momentum.  We then state a theorem for the
convergence of error through backpropagation, whose proof is presented in
Section \ref{appsect}.

\subsection{DAG architecture and backpropagation}
We now present the architecture for neural networks on DAGs. Nodes of a 
DAG can always be ordered into layers $0, \dots, L$, in which connections
(or directed edges) point to layers labeled with higher indices. We will
consider $J$ input values $x^p \in \mathbb R^{l_0}$, $ p = 1, \dots, J$.
 For each layer $i = 0, \dots, L$, there are $l_i$ nodes, with layer 0 denoting
the input. Under this ordering, define $v_{(i,j)}^{l,m}$ as the weight between
node $l$ in layer $i$ and node $m$ in layer $j$, where $i<j$. Let $v_{(i,j)}$
denote the matrix of weights from layer $i$ to $j$. Over all nodes, we use
a single (possibly nonlinear) activation function $g:\mathbb R\rightarrow
\mathbb R$  for the determination of output values. 

The explicit output values of the $l_j$ nodes in layer $j$ are denoted as
\begin{equation}
H_j = (H_j^1, \dots, H_j^{l_j}), \quad 0 \le j \le L.
\end{equation}
These are defined recursively from forward propagation, where the $j$th layer
receives input from all layers $H_i$ with $i<j$. Explicitly,
\begin{align}
&H_0 = x,  &&H_1 = g\left(H_0 v_{(0,1)} \right),\\
&H_j = g\left(\sum_{i < j} H_iv_{(i,j)} \right),  &&H_L = y = g\left(\sum_{i
<L } H_iv_{(i,L)} \right).\label{nodeeqn}
\end{align}  
 Note that here and in the future, for
a real valued function $f$, and a vector $v= (v_1, \dots, v_n)$, we will
use
the notation $f(v) = (f(v_1), \dots, f(v_n))$. Node inputs are defined as

 \begin{equation}
 S_j= \sum_{i < j} H_iv_{(i,j)}.
 \end{equation}
 
We seek to minimize the difference between a set of  $J$ desired outputs
$d^1, \dots, d^J \in \mathbb{R}^{l_L}$, and the corresponding network outputs
 $y^1, \dots, y^J \in \mathbb{R}^{l_L}$. We measure the distances between
desired and network outputs with the total quadratic error
\begin{equation}
E = \sum_{p = 1}^J \|d^p-y^p\|^2/2. \label{generrordef}
\end{equation}
The norm $\|b \|^2 = \sum b_i^2 $ denotes the usual Euclidean norm for a
vector $b = (b_1, \dots, b_n)$. Gradients of the error with respect to weights
are then defined as
\begin{equation}\label{defq}
\frac{\partial E}{\partial v_{(i,j)}^{l,m}} =  q_{(i,j)}^{l,m}.
\end{equation}

The iteration of weights by backpropagation is done through gradient descent
with momentum. Here and in the future, a superscript $k$ is used as an iteration
variable, and $\Delta X^k = X^k - X^{k-1}$ for any quantity $X$.   Weights
are updated as

\begin{equation}\label{itdef}
\Delta v^{m,l;k+1}_{(i,j)}= \tau^{m,l;k}_{(i,j)}\Delta v^{m,l;k}_{(i,j)}
- \eta q^{m,l;k}_{(i,j)}.
\end{equation}
The second term in (\ref{itdef}) corresponds to traditional backpropagation
through gradient descent, while the first term, for a predetermined $\tau
\in (0,1)$, is the contribution from adaptive momentum, with
\begin{equation}\label{taudef}
\tau^{m,l;k}_{(i,j)} = \begin{cases}\frac{\tau\|q^{m,l;k}_{(i,j)}\|}{\|\Delta
v^{m,l;k}_{(i,j)}\|} &  \|\Delta v^{m;k}_{(i,j)}\| \neq 0,\\
0  & \mathrm{otherwise}. \\
\end{cases}
\end{equation}
When the norm acts on a matrix $A = (a_{i,j})_{n\times m}$, it is treated
as the Frobenius norm, with $\|A\|^2 = \sum_{i,j} a_{i,j}^2$.  For clarity,
we sometimes place a variable denoting iteration after a semicolon to distinguish
it from node indices.  We note that a similar choice of momentum was also
used in \cite{zhang2006}.

\subsection{Convergence of backpropagation}

Our major theorem is a statement of convergence under backpropagation with
momentum.   Specifically, for some input $x^j \in l_0$, we will use a generic
desired output of $d^j \in \mathbb{R}$. We use a 1D output for clarity in
exposition.  The proof of convergence for output in multiple dimensions is
essentially the same as the one presented here. The error in this case is
then
\begin{equation}
E = \sum_{p = 1}^J \frac{(d^p-y^p)^2}2: =  \sum_{p = 1}^J \phi_p(S_L^p).
\label{edef}
\end{equation}

We will need some regularity and boundedness assumptions.  These assumptions
are similar to those used in \cite{wu2008convergence}, and may also be found
in other nonlinear optimization problems such as \cite{gori1996}. 
\begin{hyp} \label{assumes}
\hspace{0.5em}
\begin{enumerate} 
\item The function $g$, and its first two derivatives $g'$ and $g''$, are
bounded in $\mathbb R$.
\item The weights $v_{(i,j)}^k$ are uniformly bounded over layers $0 \le
i<j \le L$ and iterations $k = 1, 2, \dots.$
\item The gradient $\nabla E$ vanishes only at a finite set of points.
\end{enumerate}
\end{hyp}

It readily follows from these two assumptions that we may also uniformly
bound  $q_{(i,j)}^k, H_i^k,\phi_p, \phi_p',$ and  $\phi_p''$.  

The purpose of Assumption 3 is to establish convergence with the following
Lemma (see  \cite{sun2006}):
\begin{lemma}\label{canned}
Let $f \in C^1( \mathbb R^n, \mathbb R)$, and suppose that  $\nabla f$ vanishes
at a finite set of points.  Then, for a sequence $\{x_k\}$, if $\|\Delta x^k\|
\rightarrow 0$ and $\|\nabla f(x^k)\|\rightarrow 0$, then for some $x^* \in
\mathbb R^n$,  $x^k \rightarrow x^*$ and $\nabla f(x^*) = 0$.
\end{lemma}

\begin{theorem}\label{maint2}
Under assumptions (1) and (2), for any $s \in [0,1)$ and  $\tau = s\eta$,
there exists $C>0$ such that if 
\begin{equation}
\eta < \frac{1-s}{C(s^2+1)},
\end{equation}
then for $k = 1,2,\dots$, 
\begin{align}
&E^k = E(v_{(i,j)}^k ) \rightarrow E^* \quad 1\le i<j\le L,\\
&q^k_{(i,j)} \rightarrow 0.
\end{align}

If part (3) of the Assumptions is satisfied, weights $v_{(i,j)}^k\rightarrow
v_{(i,j)}^*$, and $E^* = E(v_{(i,j)}^*)$ is a stationary point $(\nabla E
= 0)$.
\end{theorem}

\begin{remark}
In the case of $s = 0$, Theorem \ref{maint2} is a statement convergence for
backpropagation without momentum. This can be quickly demonstrated through
gradient descent on the error function $E$. The proof for \ref{maint2} differs
from traditional gradient descent by introducing a recursive formula for
$\Delta E^k$ and $\Delta H_n^k$ given in Theorem \ref{bigtheo} which uses
the intrinsic structure of the network.
\end{remark}

The constant $C$ used is solely dependent on fixed parameters form the network,
and the uniform bounds from the Assumptions. A complete proof for Theorem
\ref{maint2} is provided in Section \ref{appsect}.

\section{Experiments} \label{sec:experiment}

\begin{figure}
\centering
\includegraphics[width=.8\textwidth]{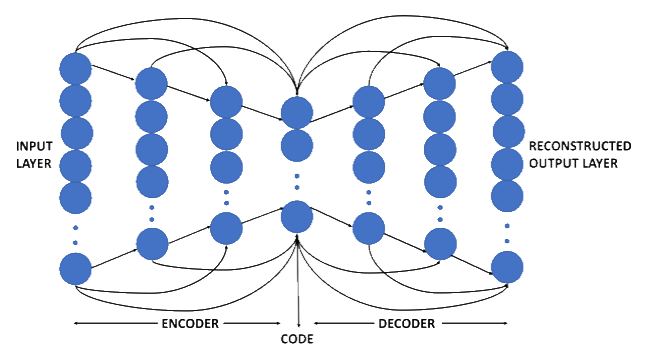}
\caption{\textbf{Architecture for CrossEncoders.}A directed edge from a node
(circle) in one layer (column of circles) to another node in a different
layer represents a connection. Additional edges  between nodes, suppressed
for presentation, may also exist. Note, however, that edges may not connect
encoding and decoding layers. }\label{autofig}
\end{figure}

In this section, we give examples of the efficacy for both DAG architectures
and the addition of momentum to backpropagation with the example of AutoEncoders,
a framework of data compression. The addition of skip connections in AutoEncoders,
entitled CrossEncoders, was studied by Agarwal et al. \cite{agarwal2018crossencoder}.
 We will apply CrossEncoders to the MNIST and Olivetti face dataset\footnote{Both
of these datasets are public, and may  may be obtained from https://www.cl.cam.ac.uk/research/dtg/attarchive/facedatabase.html
(Olivetti) and http://yann.lecun.com/exdb/mnist (MNIST)} . 

For the problem of compression, we require a code layer with index $0<\mathbf
c<L$ and dimension $l_\mathbf c < l_0$. Since we are now comparing input
and output, layer $L$ also contains $l_0$ nodes. For the error defined in
(\ref{generrordef}), we set $d^p = x^p$, and thus
\begin{equation}
E = \sum_{p = 1}^J \|x^p-y^p\|^2/2. 
\end{equation}
Since decoding should be solely dependent  from the code layer, we also require
that skip connections cannot occur between encoding layers and decoding layers.
 Thus
 \begin{equation}
 v_{(i,j)}^{l,m} = 0 \quad \hbox{if }  \quad i<\mathbf c < j. \label{encoderequire}
 \end{equation}
 See Fig. \ref{autofig} for a visual representation of the CrossEncoder architecture.

\begin{figure}
    \centering
    \includegraphics[width=.8\textwidth]{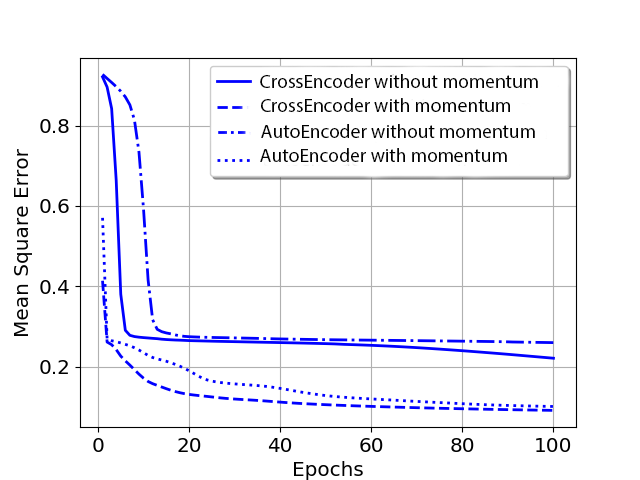}
    \caption{Effect of adding momentum to an optimizer}
    \label{fig:momentum_analysis}
\end{figure}

\subsection{Momentum analysis and performance}
\label{expt:momentum}
\noindent 
To study the effect of backpropogation with momentum, we use the standard
MNIST \cite{lecun1998gradient} dataset of handwritten digits. Each sample
 is  a $28 \times 28$ binary pixel image, transformed to a $1\times 784$ vector.
The  complete dataset consists of $70,000$ images, divided into $60,000$
training and $10,000$ testing images. We train a $784 (L_0) - 64 (L_1) -
64(L_2) - code (L_3) - 64 (L_4) - 64 (L_5) - 784 (L_6)$ network using four
different cases by considering architectures with and without skip connections,
and also backpropagation with zero and positive momentum (set to 0.95). For
the CrossEncoder, each node in layer $L_1$ is connected to all nodes in $L_3$
and each node in layer $L_4$ is connected to all nodes in $L_6$. In Fig.
\ref{fig:momentum_analysis}, we show the training loss curve for all the
four cases by plotting epochs against mean squared error. We find that for
momentum term improves the speed of convergence for architectures with and
without skip connections. Furthermore, the addition of skip connections leads
to faster convergence for both zero and positive momentum backpropagation.

\subsection{DAG architecture and performance}
The Olivetti faces dataset \cite{minear2004lifespan} is comprised of a set of
400 gray-scale face images consisting of ten different images of 40 distinct
subjects. Images for some subjects were taken with varying lighting, facial
expressions (e.g. open / closed eyes, smiling / not smiling), and facial
details (e.g. glasses / no glasses). 
The images are $64\times64$ in size and are quantized to 8-bit [0-255] scale.
A $4096 (L_0) - 500 (L_1) - 500 (L_2) - code (L_3) - 500 (L_4) - 500 (L_5)
- 4096 (L_6)$ MLP network was used for training the face dataset. Like the
previous example, each node in layer $L_1$ is connected to all nodes in $L_3$
and each node in layer $L_4$ is connected to all nodes in $L_6$.

The original $64\times64$ images were transformed to a $1\times4096$ vector.
For training, 350 images were used, and 50 images were used for the testing
dataset. Both AutoEncoders and CrossEncoders were trained for 300 epochs
using SGD optimizer with a learning rate set to 0.001 and momentum of 0.95.
For the given task, we used several lower dimension representations, such
as $1\times600$, $1\times300$, and $1\times30$, respectively. Table \ref{table8}
illustrates the performance of the respective networks for different code
size using peak signal to noise ratio (PSNR), structural similarity index
(SSIM), and normalized root mean squared error (NRMSE) metrics. In Table
\ref{table8}, we observe improved performance for CrossEncoders across all
performance metrics.

\begin{table}[ht]
\caption{PSNR, SSIM, and NRMSE values of CrossEncoder and Autoencoder between
reconstructed and the original images for Olivetti face dataset.  Higher
PSNR and SSIM values, and lower NRSME values, imply more accurate results.}
\begin{center}
\noindent\begin{tabular}{|c|c|c|c|c|c|c|}
\toprule
& \multicolumn{3}{c|}{\textbf{CrossEncoder}} & \multicolumn{3}{c|}{\textbf{Autoencoder}}
\\
\cmidrule(r){2-4}\cmidrule(l){5-7}
\textbf{Code} & \textbf{PSNR} & \textbf{SSIM} & \textbf{NRMSE}  & \textbf{PSNR}
& \textbf{SSIM} & \textbf{NRMSE} \\
\midrule
$1\times600$  & \textbf{79.4679} & \textbf{0.9040} &  \textbf{0.1464} & 75.8858
& 0.8554 &  0.2217\\ [0.5ex]
$1\times300$  & \textbf{79.4551} & \textbf{0.9046} &  \textbf{0.1467} & 75.8919
& 0.8555 &  0.2215\\ [0.5ex]
$1\times30$  & \textbf{77.8398} & \textbf{0.8791} & \textbf{0.1764} & 75.9066
& 0.8560 &  0.2211\\ [0.5ex]
\bottomrule
\end{tabular}
\end{center}
\label{table8}
\vspace{-1em}
\end{table}

\section{Proof of Convergence}\label{appsect}

\subsection{Notation and conventions}

In what follows, we will need some notation for  matrix and tensor manipulation.
First, we recall the entrywise, or Hadamard, product,  which for two matrices
$A = (a_{i,j})_{n\times m}, B = (b_{i,j})_{n\times
m}$, is defined as  $(A \circ B)_{i,j} = a_{i,j}b_{i,j}$.
By taking the sum of all entries of a Hadamard product, we obtain the Frobenius
inner product
$A:B = \sum_{i,j} (A\circ B)_{i,j}$. Also, a matrix gradient
of a real (vector) valued function is matrix (tensor) valued, with an element-wise
representation as  $\frac{\partial f }{\partial v_{(i,j)}^k} = \left(\frac{\partial
f }{\partial v_{(i,j)}^{a,b;k}}\right)_{l_i\times l_j}$ and $\frac{\partial
H_m^k }{\partial v_{(i,j)}^k}= \left(\frac{\partial H_m^{c;k} }{\partial
v_{(i,j)}^{a,b;k}}\right)_{l_i\times l_j\times l_m}$.  

In all future estimates, we look at backpropagation over a single input,
meaning $J = 1$. This allows us to suppress the variable $p$, which is essentially
done for the sake of presentation.  The proofs of Theorem \ref{bigtheo} and
Lemma \ref{simplem} generalize immediately to the case of multiple inputs
by taking sums over all inputs. Finally, in our estimates, we will use the
constant $C>0$ which depends solely on fixed parameters in the network, such
as the input value $x$, uniform bounds of node outputs $H_j$ and inputs $S_j$,
and for generalizing to multiple inputs, the size of the dataset $J$. The
constant $C$ is used in multiple estimates, and may increase each time it
appears.

 \subsection{Estimates on node output increments}

 Our  major technical theorem shows that the increments of outputs $\Delta
H_n^{k+1}$   are
similar, up to first order, to $Q^k(H^k)$, where  $Q^k$ denotes the differential
operator
\begin{equation}\label{qeqn}
Q^k =\sum_{i<j\le L}  \Delta v_{(i,j)}^{k+1}:\frac{\partial }{\partial v_{(i,j)}^{k}}.
\end{equation}  
Note that when $Q^k$ acts on a  length $l_m$ vector, the matrix inner product
in (\ref{qeqn})
is between a $l_i \times l_j$-sized matrix and a $l_i\times l_j \times l_m$-sized
tensor, and is a vector of size $l_m$.

The major utility of introducing $Q^k$ is that it provides a simple bound
when acting on $E^k$.  Specifically, using (\ref{defq}) , (\ref{itdef}),
and (\ref{taudef}), it is straightforward to show
\begin{equation}
Q^k(E^k)\le (-\eta+\tau)\sum_{i<j \le L}\|q_{(i,j)} ^k\|^2. \label{midest}
\end{equation}

\begin{theorem}\label{bigtheo}

There exists a universal constant $C>0$ such that
\begin{equation}
|Q^k(E^k)-\Delta E^{k+1}| \le C\left(\sum_{n\le L}\|\Delta H_n^{k+1}\|^2
+ \sum_{m<n\le
L}\|\Delta v_{(m,n)}^{k+1}\|^2\right). \label{th1bnd}
\end{equation}
\end{theorem}

\begin{proof} We show (\ref{th1bnd}) follows through three steps: (1) finding
a recurrence relation, with respect to the ordering of hidden layers, for
$Q^k(H^k_n)$ and $Q^k(E^k)$; (2) finding a similar relation for $\Delta H_n^{k+1}$
and $\Delta E^{k+1}$; and (3) comparing the two relations.
 
 \textit{(1) (A recurrence for $Q^k(H_n^k)$ and $Q^k(E^k)$).} Applying the
chain rule to the total error (\ref{edef}), using (\ref{nodeeqn}),
 and rearranging sums, \begin{align}\label{crule1}
Q^k(E^{k}) = \phi'(S_L^k)\sum_{i<j\le L} \Delta v_{(i,j)}^{k+1}:
\frac{\partial }{\partial v_{(i,j)}^k}\left(\sum_{m<L}H_m^{k}  v_{(m,L)}^k
 \right)\\  =   \phi'(S_L^k)\sum_{m<L}\sum_{i<j\le L} \Delta v_{(i,j)}^{k+1}:
\frac{\partial }{\partial v_{(i,j)}^k}\left(H_m^{k} v_{(m,L)}^k \right) \label{qeform}.
\end{align}

We now focus on expressing (\ref{qeform}) in a recursive form. We begin with
considering the terms in (\ref{qeform}) with $j = L$. We first work elementwise
by differentiating with respect to  the $(a,b)$ entry of the matrix derivative
for  $\frac{\partial }{\partial v_{(i,j)}^k}\left(H_m^{k}v_{(m,L)}^k \right)$.
From the product rule,  this can be written as a sum of vectors, with
\begin{align} \label{msplit}
&\frac{\partial }{\partial v_{(i,L)}^{a,b;k}}\left( H_m^{k}v_{(m,L)}^k \right)
= \frac{\partial  H_m^{k} }{\partial v_{(i,L)}^{a,b;k}}v_{(m,L)}^k + H_m^{k}
\frac{\partial v_{(m,L)}^k }{\partial v_{(i,L)}^{a,b;k}}\\
  &=  \frac{\partial  H_m^{k} }{\partial v_{(i,L)}^{a,b;k}}v_{(m,L)}^k +
(0, \dots,\underbrace{ \delta_{i,m}H_m^{a;k}}_{\hbox{ $b^{th}$ entry}}, \dots,
0)\\
&:=A_{i,m}^{a,b;k}+B_{i,m}^{a,b;k}.  \label{msplit2}
\end{align}

Each of these terms is handled in turn.  First, summing the Frobenius inner
product of the matrix $\Delta v_{(i,L)}^{k+1}$ and the tensor $A_{i,m}^k$,
we may write
\begin{align}
\sum_{m<L}\sum_{i< L} \Delta v_{(i,L)}^{k+1}:A_{i,m}^k 
  &= \sum_{m<L} \sum_{i<L} \sum_{\substack{a<l_i\\b<l_L}}\Delta v_{(i,L)}^{a,b;k+1}\frac{\partial
 H_m^{k} }{\partial v_{(i,L)}^{a,b;k}}v_{(m,L)}^k\\
&= \sum_{m<L} \sum_{i< L} \left(\Delta v_{(i,L)}^{k+1}:\frac{\partial H_m^k
}{\partial
v_{(i,L)}^k}\right)v^k_{(m,L)}.
 \label{nondiag}
\end{align}

For $B_{i,m}^k$, we also work elementwise, and write the Frobenius inner
product as
\begin{align}
&\sum_{m<L}\sum_{i< L} \Delta v_{(i,L)}^{k+1}:B_{i,m}^k
  = \sum_{m<L} \sum_{\substack{a\le l_m\\b\le l_L}}\Delta v_{(m,L)}^{a,b;k+1}B_{m,m}^{a,b;k}\\
     &=\sum_{m<L}\left(\sum_{a\le l_m}H_m^{a;k}\Delta v_{(m,L)}^{a,1;k+1},
\dots,
\sum_{a\le l_m}H_m^{a;k}\Delta v_{(m,L)}^{a,l_L;k+1}\right)\\&=\sum_{m<L}H_m^k\Delta
v_{(m,L)}^{k+1}.\label{diag}
\end{align}

Calculations for double sum in (\ref{qeform}) for the remaining terms with
$j<L$ are similar to the case $j = L$, except that there is
no corresponding $B_{i,m}^k$ term. Indeed, we can show
\begin{align} 
\sum_{m<L}\sum_{i<j< L} \Delta v_{(i,j)}^{k+1}:
\frac{\partial }{\partial v_{(i,j)}^k}\left(H_m^{k}v_{(m,L)}^k \right) 
 = \sum_{m<L}\sum_{i<j< L}\left(\Delta v_{(i,j)}^{k+1}:\frac{\partial H_m^k}{\partial
v_{(i,j)}^k}\right)v_{(m,L)}^k. \label{lessj} \end{align}
Putting together (\ref{crule1})-(\ref{lessj}), we arrive at
  \begin{align}
 &\sum_{m<L}\sum_{i<j\le L} \Delta v_{(i,j)}^{k+1}:
\frac{\partial }{\partial v_{(i,j)}^k}\left( H_m^{k} v_{(m,L)}^k \right)
\label{bigsumrewrite}\\
&= \sum_{m<L} \left( H_m^k  \Delta v_{(m,L)}^{k+1} 
+\sum_{i<j\le m}\left(\Delta v_{(i,j)}^{k+1}:\frac{\partial H_m^k}{\partial
v_{(i,j)}^k}\right)v_{(m,L)}^k\right) \label{firstrep}\\
 &= \sum_{m<L}\left(H_m^k  \Delta v_{(m,L)}^{k+1} +Q^k(H^k_m)  v^k_{(m,L)}
 \right).
\end{align}
Note that (\ref{firstrep}) uses the fact that since $H_m^k$ only depends
on layers $1$ through $m-1$, we may truncate the sum of $Q^k$ and write\begin{equation}
Q^k(H_m^k) =\sum_{i<j\le m}  \Delta v_{(i,j)}^{k+1}:\frac{\partial H_m^k
}{\partial v_{(i,j)}^k}.
\end{equation}
We may now substitute (\ref{bigsumrewrite}) into (\ref{qeform}) to yield
the recursive formula
\begin{align}
Q^k(E^k)=  \phi'\left(S_L^k \right)\sum_{m<L}\left( H_m^k\Delta v_{(m,L)}^{k+1}+Q^k(H^k_m)v^k_{(m,L)}
  \right).
\end{align}
From similar calculations, the formula over a node $H^k_n$, with $n<L$, is
\begin{align}
Q^k(H^k_n)= g'\left(S_n^k\right)\circ\sum_{m<n}\left( H_m^k\Delta v_{(m,n)}^{k+1}+Q^k(H^k_m)v^k_{(m,n)}
  \right).\label{qrec}
\end{align}

\textit{(2) (A recurrence for $\Delta H_n^{k+1}$ and $\Delta E^{k+1}$).}
A recursive formula for   $\Delta H^{k+1}_n$ is found through a Taylor expansion
of $E(S_L^{k+1})$ centered at $S_L^k$. Specifically, there exists $t_{k}$
between $S_L^k$
and $S_L^{k+1}$ with
\begin{align}
\Delta E^{k+1} &=  \phi'\left(S_L^{k} \right)\Big(\sum_{m < L} \Delta( H_m^{k+1}v_{(m,L)}^{k+1}
) \Big)+\frac 12 \phi''(S_L^k)\left(\sum_{m < L}\Delta(  H_m^{k+1}v_{(m,L)}^{k+1})\right)^2\\
 &=  \phi'\left(S_L^{k} \right)\sum_{m < L} \left( \Delta H_m^{k+1}
v_{(m,L)}^k +H^k_m\Delta v_{(m,L)}^{k+1}
+ \Delta H_m^{k+1}\Delta v_{(m,L)}^{k+1} \right)
\\&+{\frac 12} \phi''(t_{k})\left(\sum_{m < L}\Delta( H_m^{k+1} v_{(m,L)}^{k+1})\right)^2.
\end{align} 
Similarly, {there exist $t_{n,k} = (t_{n,k}^{1}, \dots, t_{n,k}^{l_l})$ where
each
$t_{n,k}^r$ lies between $S_n^{r;k}$ and $S_n^{r;k+1}$ for $r = 1, \dots
l_n$} and
\begin{align}
\Delta H_n^{k+1} = &g'\left(S_n^k \right)\circ\sum_{m < n}\left (\Delta H_m^{k+1}v_{(m,n)}^k
 +H^k_m\Delta v_{(m,n)}^{k+1}
+\Delta H_m^{k+1}\Delta v_{(m,n)}^{k+1}  \right) \label{diffrec0}
\\+&{\frac 12g''(t_{n,k})\circ\left(\sum_{m < n}\Delta (H_m^{k+1} v_{(m,n)}^{k+1})
\right)^2}\label{diffrec}. 
\end{align}

\textit{(3) (Comparing recurrences).}
From (1) and (2) of Assumptions \ref{assumes}, we may derive the simple bound

\begin{equation}
\|\Delta H_m^{k+1}\Delta v_{(m,n)}^{k+1}\| \le C( \|\Delta v_{(m,n)}^{k+1}\|^2+\|\Delta
H_m^{k+1}\|^2)
\end{equation}
for some constant $C>0$. 
Taking differences of (\ref{diffrec}) and (\ref{qrec}),
for any $n<L$,  we then obtain the recurrence inequality
\begin{align}
&\|Q^k(H_n^{k})-\Delta H_n^{k+1} \|\le C\left(\sum_{m < n}\|Q^k(
H_m^{k})-\Delta H_m^{k+1} \|  \right)  \label{rec1}\\&+ C\sum_{m<n} ( \|\Delta
v_{(m,n)}^{k+1}\|^2+\|\Delta
H_m^{k+1}\|^2) \label{mainrec1}.
\end{align}
Replacing $H_n^k$ with $E^k$ in (\ref{rec1}) produces the same type of inequality,
with the sum in (\ref{mainrec1}) now ranging from $m = 1, \dots, L-1$. Repeated
applications of (\ref{mainrec1}) to  $E^k$ and subsequently to $H_n^k$,
for $n = 1, \dots, L-1$, result in   
\begin{align}
&|Q^k(E^k)-\Delta E^{k+1}|\le C\left(\|Q^k( H_0^{k})-\Delta H_0^{k+1} \|
 \right)
\\&{+ C\left(\sum_{n< L}\|\Delta H_n^{k+1}\|^2 + \sum_{m<n< L}\|\Delta v_{(m,n)}^{k+1}\|^2\right)}
\label{mainrec}.
\end{align}
To complete the proof, we note that
 the input data $x$ does not change under iterations, so \begin{equation}
Q^k( H_0^{k})-\Delta H_0^{k+1} \equiv 0.
\end{equation}
\end{proof}

 We now bound the quadratic terms in (\ref{th1bnd}).

\begin{lemma}\label{simplem}
For some constant $C>0$,

\begin{enumerate}
\item
\begin{equation}
 \|\Delta v^{k+1}_{(m,n)}\| \le (\eta+\tau)\|q_{(m,n)}^k\|. \label{fstpart}
\end{equation}
\item 
\begin{equation}
 \|\Delta H^{k+1}_n\| \le C(\eta+\tau)\sum_{i<j<n}\|q_{(i,j)}^k\|. \label{scdpart}
\end{equation}
\end{enumerate}
\end{lemma}

\begin{proof}
We may show (\ref{fstpart}) immediately from (\ref{itdef}) and (\ref{taudef}).
 For (\ref{scdpart}), we use strong induction, assuming the inequality holds
for layers $m < n$ (note that the base case holds trivially for $n = 0$).
From the Taylor expansion used in (\ref{diffrec0})-(\ref{diffrec}), and the
boundedness of $g'$ and  $g''$:
 \begin{align}
\|\Delta H^{k+1}_n\| \le &C \sum_{m < n}\left \|\Delta H_m^{k+1}v_{(m,n)}^k
 +H^k_m\Delta v_{(m,n)}^{k+1}
+\Delta H_m^{k+1}\Delta v_{(m,n)}^{k+1}  \right
 \| \label{linepart}\\ &+C\left \|\left(\sum_{m <
n}\Delta (H_m^{k+1}v_{(m,n)}^{k+1} )\right)^2\right \|.\label{quad}
\end{align}
 From the induction hypothesis, the boundedness of weights and node outputs,
and (\ref{fstpart}), the right hand side of (\ref{linepart}) is bounded by
the right hand side of (\ref{scdpart}) for some $C>0$.  From the boundedness
assumptions,  
\begin{align}
\left \|\left(\sum_{m <n}\Delta (H_m^{k+1}v_{(m,n)}^{k+1} )\right)^2\right \| \le C \left \|\sum_{m<n}\Delta (H_m^{k+1}v_{(m,n)}^{k+1} )\right \|,
\end{align}
which implies that we may use similar estimates for (\ref{quad}) which
we used in (\ref{linepart}) to arrive at (\ref{scdpart}).
\end{proof}
From Theorem \ref{bigtheo}, (\ref{midest}), and Lemma \ref{simplem},
the iteration of error may now be estimated as  
\begin{align}
\Delta E^{k+1} &\le  C\left(\sum_{n\le L}\|\Delta H_n^{k+1}\|^2 + \sum_{m<n\le
L}\|\Delta v_{(m,n)}^{k+1}\|^2\right)+Q^k(E^k)
\\&\le \left(-\eta+\tau+C(\tau^2+\eta^2)\right) \sum_{m<n \le L} \|q_{(m,n)}^k\|^2.
\label{graddif}
\end{align}

\subsubsection{Proof of convergence}

For some  $s \in [0,1)$, assume  $\tau = s\eta $.  It is straightforward
to show that the term in
front of the norms in (\ref{graddif}) is negative when
\begin{align}
\eta < \frac{1-s}{C(s^2+1)}.
\end{align}
Under this constraint, $E^k$ is decreasing under each iteration. The summability
for $\|q_{(i,j)}^k\|^2$ also follows, since
\begin{align}
\sum_{k = 1}^\infty \|q_{(i,j)}^k\|^2 \le\frac 1{\left(\eta-\tau- C(\tau^2+\eta^2)\right)}
\sum_{k = 1}^\infty \Delta E^k <\infty.   
\end{align}
Thus  $\|q_{(i,j)}^k\|\rightarrow 0$ and, from (\ref{fstpart}),  $ \|\Delta
v_{(i,j)}^k\|\rightarrow 0.$
Lemma \ref{canned}
and part (3) of Assumption 2 imply a set of minimum weights $v^*_{1}, v_2^*,
z^*, w^*$,
which determine a  {stationary point} of $E$. This shows Theorem \ref{maint2}.

\section{Conclusion}
We have studied a feed-forward network with skip-layer connections. The possible
directed graph architectures are the class of directed acyclic graphs.  As
shown in \cite{huang2016densely}, introducing skip connections often increases
the performance of a deep neural network. In \cite{agarwal2018crossencoder}
and in Section \ref{sec:experiment}, we have demonstrated increased performance
in the setting of AutoEncoders. For our main result, we have established
the convergence of backpropagation with adaptive momentum of networks with
skip-connections. This generalizes the result of Wu et al. \cite{wu2008convergence}
who established convergence for a feed forward network with one hidden layer.
While we have considered general DAG architectures, it remains to investigate,
both through theory and experiment, the optimality properties with regards
to the number of layers and  skip connections. We hope to address these properties
in future works.

\bibliographystyle{siam}
\bibliography{references}

\end{document}